\documentclass{article}

\usepackage[utf8]{inputenc}
\usepackage[framemethod=tikz]{mdframed}
\usepackage{amsmath,amsthm,amssymb,epsfig, psfrag, titlesec}
\usepackage[margin=1in]{geometry}
\usepackage{mathtools}
\usepackage{amsfonts}
\usepackage{textcomp}
\usepackage{pgfplots}
\usepackage{graphicx}
\usepackage{wrapfig}
\usepackage[]{youngtab}
\usepackage{amsthm}
\usepackage{tikz}
\usepackage{tikz-cd}
\usepackage{stmaryrd}
\usepackage{enumerate}
\usepackage{bbm}
\usepackage{float}

\usepackage{eucal}
\usepackage{mathrsfs}
\pgfplotsset{width=10cm,compat=1.9}
\usepackage{hyperref}
\hypersetup{
    colorlinks=true,
    linkcolor=blue,
    citecolor=blue
}

\usepackage{algorithm}
\usepackage{verbatim}
\usepackage[noend]{algpseudocode}
\usepackage{xfrac}

 \usepackage{thmtools}
\usepackage{thm-restate}

\declaretheorem[name=Lemma]{lemma}

\declaretheorem[name=Assumption]{assumption}

\declaretheorem[name=Remark,style=definition]{remark}
\declaretheorem[name=Proposition]{proposition}

\makeatletter
  \renewenvironment{proof}[1][Proof]%
  {%
   \par\noindent{\bfseries\upshape {#1.}\ }%
  }%
  {\qed\newline}
\makeatother

\titleformat{\section}
  {\normalfont\sffamily\large\bfseries\centering\uppercase}
  {\thesection.}{.5em}{}
\titleformat{\subsection}
  {\normalfont\sffamily\bfseries}
  {\thesubsection}{.5em}{}

\newtheoremstyle{TH1}
  {\topsep}%
  {\topsep}%
  {\normalfont}%
  {}%
  {\bfseries}% 
  {:}%
  {.5em}%
  {\thmname{#1}\thmnote{~(#3)}}%
\theoremstyle{TH1}

\newmdtheoremenv[
skipabove=\baselineskip,
skipbelow=\baselineskip,
hidealllines=true,
innertopmargin=6pt,
linewidth=4pt,
linecolor=gray!40,
backgroundcolor=gray!8
]{exmpl}{{\sf Example}}

\newcommand{\algoname}{\textsf{SquareCB.Submodular}}

\newcommand{\norm}[1]{\left\lVert#1\right\rVert}

\definecolor{lightgray}{gray}{0.5}

\newcommand{\cB}{\mathcal{B}}
\newcommand{\cC}{\mathcal{C}}
\newcommand{\cD}{\mathcal{D}}
\newcommand{\cF}{\mathcal{F}}

\newcommand{\cH}{\mathcal{H}}
\newcommand{\cI}{\mathcal{I}}
\newcommand{\cM}{\mathcal{M}}

\newcommand{\cT}{\mathcal{T}}
\newcommand{\cW}{\mathcal{W}}
\newcommand{\cX}{\mathcal{X}}

\newcommand{\reals}{\mathbb{R}}
\newcommand{\pred}{\widehat{y}}

\newcommand{\En}{\mathbb{E}}
\newcommand{\inner}[1]{\left\langle #1 \right\rangle}

\newcommand{\igw}{\mathsf{IGW}}
\renewcommand{\S}{\mathsf{S}}
\newcommand{\T}{\mathsf{T}}
\newcommand{\C}{\mathsf{C}}
\newcommand{\Shat}{\widehat{\S}}
\newcommand{\util}{\boldsymbol{u}}
\newcommand{\utilhat}{\widehat{\util}}

\newcommand{\vtil}{\boldsymbol{v}}

\newcommand{\tr}{\ensuremath{{\scriptscriptstyle\mathsf{\,T}}}} 

\newcommand{\Regsq}{\mathsf{Reg_{sq}}(\cF,n)}

\newcommand\blfootnote[1]{%
  \begingroup
  \renewcommand\thefootnote{}\footnote{#1}%
  \addtocounter{footnote}{-1}%
  \endgroup
}

\title{On Submodular Contextual Bandits}
\author{Dean P. Foster$^{\ast}$ \and Alexander Rakhlin$^{\dagger, \ast}$}

\begin{document}
\maketitle

\begin{abstract}
    We consider the problem of contextual bandits where actions are subsets of a ground set and mean rewards are modeled by an unknown monotone submodular function that belongs to a class $\cF$. We allow time-varying matroid constraints to be placed on the feasible sets. Assuming access to an online regression oracle with regret  $\Regsq$, our algorithm efficiently randomizes around local optima of estimated functions according to the Inverse Gap Weighting strategy \cite{abe1999associative, foster2020beyond}. We show that cumulative regret of this procedure with time horizon $n$ scales as $\mathcal{O}(\sqrt{n \Regsq})$ against a benchmark with a multiplicative factor $1/2$. On the other hand, using the techniques of \cite{filmus2014monotone} we show that an $\epsilon$-Greedy procedure with local randomization attains regret of $\mathcal{O}(n^{2/3} \Regsq^{1/3})$ against a stronger $(1-e^{-1})$ benchmark.
\end{abstract}

\section{Introduction}
\blfootnote{$^\ast$Amazon, $^\dagger$MIT}

In this short note, we consider the problem of contextual bandits with submodular rewards. On round $t=1,\ldots,n$, the learner observes context $x_t\in\cX$, chooses (in a randomized fashion) a subset $S_t\subseteq [A] \triangleq \{1,\ldots, A\}$ of size at most $k$, and observes a random reward $r_t$ with mean $\util^*(\S_t,x_t)$, for some unknown $\util^*\in\cF\subseteq \cF_{\textsf{nms}}$. Here  $\cF_{\textsf{nms}}$ is the set of all functions $\util:2^{[A]}\times \cX\to [0,1]$ such that $\util(\cdot, x)$ is nonnegative monotone submodular for each $x\in\cX$, and the model class $\cF\subseteq \cF_{\textsf{nms}}$ is known to the learner. We assume for simplicity that $r_t\in[0,1]$ almost surely.  In addition, we posit time-varying matroid constraints $\cM_t$ on the feasible sets $\S_t$, and allow these constraints to be revealed to the decision-maker at the beginning of each round. 
The goal of the learner is to minimize (with high probability) the regret
\begin{align}
    \label{eq:regret_def}
     \sum_{t=1}^n c\cdot \max_{\S^*_t} \util^*(\S^*_t,x_t) - \util^*(\S_t,x_t), 
\end{align}
the difference between a multiple of cumulative mean rewards of choosing the best subsets $\S_t^*$ of size $k$ (subject to matroid constraints omitted here for simplicity) and the mean utilities of the subsets chosen by the decision-maker, for a constant $c\leq 1$. We assume that $x_1,\ldots,x_n$ is an arbitrary sequence of contexts that may be chosen by Nature adaptively. Furthermore, as in \cite{foster2020beyond}, we assume access to an online regression oracle, as specified in the next section.

The setting can naturally model situations such as product selection at a fulfillment center, product search, news recommendation, or advertisement. In product search, there is a diminishing utility of adding another item to the list of responses to the search query, naturally modeled by the notion of submodularity. Similarly, adding an item to a stock at a fulfillment center has diminishing utility when similar items are already present.

In our formulation, the decision-maker does not have access to the function $\util^*$ and needs to balance  learning this function in the presence of a potentially high-dimensional context together with maximizing it. This exploration-exploitation tradeoff is well-understood for contextual bandits, but the exploration step is somewhat more delicate when actions are subsets of a ground set. Indeed, a na\"ive application of algorithms for contextual bandits incurs a computational and statistical cost of $\mathcal{O}\left( A^k \right)$. We show that this prohibitive cost can be avoided by an argument that leverages the structure of submodular functions.

The widely-studied problem of contextual bandits corresponds to $k=1$, while the top-$k$ problem studied in \cite{sen2021top} corresponds to a linear (or, modular)  $\util^*$; both of these minimize regret for $c=1$. For $c<1$, the objective \eqref{eq:regret_def} studied in this paper is significantly less ambitious as compared to that of \cite{pmlr-v40-Rakhlin15}; the latter paper exhibited a relaxation technique (in a non-bandit setting) for efficiently achieving---under certain assumptions---vanishing regret with $c=1$ despite NP hardness of the offline problem.

The first method in the present paper is an application of the \textsf{SquareCB} algorithm \cite{foster2020beyond}, with a few twists that exploit the structure of submodular functions. Let  $\utilhat_t:2^{A} \times \cX \to \reals$ denote our estimate of the utility at time $t$, constructed in an online manner, as described below. Given $x_t$, we compute a local maximum $\Shat_t$ of this function, in the sense that no swap of elements increases the (estimated) utility. We then randomize around this local maximum with the Inverse Gap Weighting strategy \cite{abe1999associative, foster2020beyond}, with actions being the \textit{local swap operations}. The regret guarantee of \textsf{SquareCB} with time-varying action sets then implies an approximate local optimum of the true utility $\util^*$, in a cumulative sense; submodularity of the aggregate utility then guarantees an approximation ratio for this solution. Our method naturally generalizes that of \cite{sen2021top} for linear utility, where the $k-1$ actions are chosen greedily and the last action is randomized among the remaining $A-(k-1)$ choices.

To make the setting more widely applicable, we allow matroid constraints: for each time step $t$, we insist that the allowed sets $\S_t$ are independent sets of a matroid $\cM_t=([A], \cI_t)$ defined over the common ground set $[A]$. The matroid $\cM_t$ is revealed together with context $x_t$ (and can be thought of as part of the context), and can be chosen arbitrarily and adaptively by Nature. As a simple example, consider the \textit{partition matroid} $\cM_t$ defined as follows. For a collection $\{P_i\subseteq [A]\}_{i=1}^m$ of disjoint sets and capacities $c_i\in[0,|P_i|]$, independent sets are those with intersection of size no more than $c_i$ with each set $P_i$; that is, $\cI_t=\{I\subseteq \cup_i P_i: |I\cap P_i|\leq c_i, i\in[m]\}$. The partition matroid can model, for instance, constraints on choosing multiple items from the same category for product search, and the sets $P_i$ and the capacities can themselves be chosen adaptively by Nature. We remark that our results can be further generalized to time-varying action sets and multiple matroid constraints per time step, but we leave these generalization out for simplicity of presentation. 

We end this section with a specific model for personalized user engagement (see also \cite{yue2008predicting, yue2011linear,radlinski2008learning,kveton2015cascading, zong2016cascading, li2019online} for related models).

\begin{exmpl}
    Given a permutation $\pi$ of $A$ items, user engagement is modeled in \cite{asadpour2020ranking,niazadeh2021online} as
    \begin{align}
        F(\boldsymbol{\pi}) = \sum_{i=1}^{A} \lambda_i f_i(\{\pi_1,\pi_2,\ldots,\pi_i\})
    \end{align}
    where $\pi_i$ is the $i$th element in the permutation $\boldsymbol{\pi}$, $\lambda_i\geq 0$ are known coefficients, and $f_i$ are known non-negative monotone submodular functions. Here $\lambda_i$ is interpreted as the probability that the (randomly drawn) user inspects items up to position $i$ and $f_i$ is the expected probability of a click given that the user inspects items up to position $i$.
    
    Consider a contextual version of this problem: to maximize (in an online manner, given a sequence of contexts) the function
    \begin{align}
        F(\boldsymbol{\pi}, x) = \sum_{i=1}^{A} \lambda_i(x) f_i(\{\pi_1,\pi_2,\ldots,\pi_i\},x),
    \end{align}
    where the functions $\lambda_i, f_i$ are to be learned online. For any $x$, the utility $F(\boldsymbol{\pi},x)$ can be viewed as a submodular function over a domain of size $A^2$ with a laminar matroid constraint, as described in \cite{asadpour2020ranking}. More precisely, the matroid is defined by taking $M_1,\ldots,M_A$ as $M_s=\{(i,j): i\leq s, 1\leq j\leq A\}$ and the independent sets $I$ as those satisfying the capacity constraint $|I\cap M_s| \leq s$. Then the submodular function corresponding to $F(\boldsymbol{\pi},x)$ can be defined, with respect to the ground set of elements  $[A]\times[A]$, as
    \begin{align}
		\label{eq:example_ranking}
		\util^*(S,x) = \sum_{i=1}^{A} \lambda_i(x) f_i(T_i,x), ~~~~ S\subseteq [A]\times [A],
	\end{align}
    where $T_i=\{j ~|~ \exists k\leq i ~\text{s.t.}~ (k,j)\in S\} $.
\end{exmpl}

\paragraph{Prior work}

In the past two decades, submodularity has been successfully employed in a great number of machine learning domains \cite{krause2007near, yue2008predicting,  leskovec2007cost, lin2012learning, krause2014submodular}. Linear submodular bandits have been studied in \cite{yue2011linear} in a different feedback model where the set is built progressively with noisy feedback at each stage (see also the generalization to RKHS in \cite{chen2017interactive}). Adaptive submodular bandits have been studied in \cite{gabillon2013adaptive, esfandiari2021adaptivity}, yet the setting is different from the one studied in this paper. In contrast to the aforementioned papers, in the present formulation the decision-maker selects a set and receives noisy feedback for this single choice. In particular, this means that the estimate $\utilhat_t$ of the submodular function at smaller subsets may be unreliable. Additional difficulty stems from the fact that UCB-style methods cannot be employed for contextual bandits with general function approximation \cite{foster2020beyond}, and closeness of the regression estimate $\utilhat_t$ and the unknown model $\util^*$ cannot be ensured uniformly for all sets even with the method of \cite{foster2020beyond}. 

To this end, we would like to highlight again our proof technique, motivated by no-internal-regret statements in online learning: by randomizing around a local optimum of our estimated function $\utilhat_t$ according to the IGW strategy, we obtain a regret guarantee which itself can be interpreted as a statement about the degree to which the chosen sets are, on average, local maxima of $\util^*$. We remark that this simple  approach can be used for nonconvex contextual bandit problems beyond submodularity, as long as local maxima are guaranteed to be within a multiplicative factor from the optimum.

This paper is organized as follows. In Section~\ref{sec:structure}, we state the main structural results that imply global approximation ratios for local optima. In Section~\ref{sec:algos}, we present the two contextual bandit algorithms. Finally, in Section~\ref{sec:examples}, we develop efficient online regression oracles for several models.

\section{Structure of Local Optima}
\label{sec:structure}

Recall that a matroid $\cM=([A], \cI)$ is defined via a collection $\cI=\cI(\cM)$ of independent sets. Given $\S\in\cI$, we denote by $\cI(\S)\subseteq \cI$ the set of independent sets obtained by swapping one element out and a new element in, together with the set $\S$:
\begin{align}
\label{eq:indep_set}
    \cI(\S)=\left\{\S'\in\cI: |\S'|=|\S|, |\S\setminus \S'|= |\S'\setminus \S|= 1\right\}\cup \{\S\}.
\end{align}
Let $\cB(\cM)$ be the set of bases (maximal independent sets) of $\cM$. Similarly to \eqref{eq:indep_set}, define $\cB(\S)\subseteq \cB(\cM)$ to be the collection of bases of the matroid that differ from $\S$ in at most one element. Recall that all the bases have the same cardinality, known as the rank of the matroid. 

We remark that in the case of no matroid constraint, we take $\cI$ as the set of all sets and $\cI(\S)$ as the set of all sets of the same cardinality as $\S$ that differ in at most one element.

The following Lemma follows from the classical results on monotone submodular optimization with matroid constraints (\cite{fisher1978analysis}, see also \cite[Lemma 2]{lee2009non}).
\begin{lemma}
\label{lem:local_opt}
    Let $\util_1,\ldots,\util_n$ be submodular nonnegative monotone functions and $\cM_j=([A],\cI_j)$, $j=1,\ldots,n$, be matroids on $[A]$. Let $(\S_1,\ldots,\S_n)$,  with $\S_j\in\cI_j$ and $|\S_j|=k$,  be a joint local $\varepsilon$-optimum, in the sense that
    \begin{align}
\label{eq:def:joint_opt}
    \sum_{j=1}^n \util_j(\S_j) \geq \sum_{j=1}^n \max_{\T_j \in \cI_j(\S_j)} \util_j (\T_j) - n\varepsilon.
\end{align}
    Then
    \begin{align}
        \label{eq:optimality_of_loc}
        \sum_{j=1}^n \max_{\T_j \in \cI_j(\S_j)} \util_j (\T_j) \geq \sum_{j=1}^n \util_j(\S_j) \geq \frac{1}{2}\left(\sum_{j=1}^n \max_{\T_j\in\cI_j} \util_j(\T_j) -nk\varepsilon \right).
    \end{align}
\end{lemma}

Lemma~\ref{lem:local_opt} guarantees a $1/2$ approximation ratio for the ``on average'' local optima under the matroid constraints. We now turn to the techniques developed in 
\cite{filmus2014monotone, sviridenko2017optimal} for improving this ratio to $1-1/e$ at the cost of requiring more information about the structure of the submodular function, and in particular its values on smaller subsets of a given set of interest. Define an operator $\cT$ that sends a submodular function $\vtil:2^{[A]}\to \reals$ to a submodular function $\cT \vtil:2^{[A]}\to \reals$ defined as a weighted sum of $\vtil$ on subsets:
\begin{align}
    \cT \vtil (\S) = \sum_{\T\subseteq \S} w_{|\S|,|\T|} \cdot \vtil(\T),
    ~~~ w_{s,t} = \int_{0}^1 \frac{e^p}{e-1} p^{s-1} (1-p)^{s-t} dp, ~~~ s\geq t \geq 1.
\end{align}
Letting $\tau(S) = \sum_{\T\subseteq \S}  w_{|\S|,|\T|}, $
the normalization factor that only depends on the  cardinality of $\S$, we can define a distribution $\cD_\S$ supported on subsets $\T\subseteq \S$ such that  
\begin{align}
	\label{def:dist_subsets}
    \cD_\S \left(\T\right)= \frac{w_{|\S|,|\T|}}{\tau(\S)}\cdot\mathsf{1}\{\T\subseteq \S\}.
\end{align} 
Then, trivially,
\begin{align}
	\label{eq:def_cTvtil}
    \cT \vtil(\S) = \tau(\S) \cdot \En_{\T\sim \cD_\S} \vtil(\T).
\end{align}

With this notation in place, we have the following version of Lemma~\ref{lem:local_opt}:
\begin{lemma}
    \label{lem:local_opt_T}
    Let $\util_1,\ldots,\util_n$ be submodular nonnegative monotone functions and $\cM_j=([A],\cI_j)$, $j=1,\ldots,n$, be matroids on $[A]$. Let $(\S_1,\ldots,\S_n)$,  with $\S_j\in\cB_j(\cM_j)$ and $|\S_j|=k$,  be a joint local $\varepsilon$-optimum with respect to $\cT\util_j$'s, in the sense that
    \begin{align}
\label{eq:def:joint_opt_T}
    \sum_{j=1}^n \cT\util_j(\S_j) \geq \sum_{j=1}^n \max_{\T_j \in \cB_j(\S_j)} \cT\util_j (\T_j) - n\varepsilon.
\end{align}
    Then
    \begin{align}
        \label{eq:optimality_of_loc_T}
        \sum_{j=1}^n \max_{\T_j \in \cB_j(\S_j)} \util_j (\T_j) \geq \sum_{j=1}^n \util_j(\S_j) \geq (1-e^{-1})\left(\sum_{j=1}^n \max_{\T_j\in\cB_j} \util_j(\T_j) -nk\varepsilon \right).
    \end{align}
\end{lemma}

\section{Contextual Bandit Algorithms}
\label{sec:algos}

In this section, we present two algorithms that attain sublinear regret in \eqref{eq:regret_def} whenever the associated online regression problem has sublinear regret. The first algorithm, based on \textsf{SquareCB}, achieves regret with rate $\mathcal{O}(\sqrt{n \Regsq})$ and constant $c=1/2$, while the second method, $\epsilon$-Greedy, attains the slower $\mathcal{O}(n^{2/3} \Regsq^{1/3})$ regret rate but against a stronger benchmark with $c=1-e^{-1}$. Whether the first or the second result is stronger overall depends on the relative scales of the benchmark growth and $n$.\footnote{In practice one may interpolate between \textsf{SquareCB} and $\epsilon$-Greedy with a single extra parameter multiplying $2K$ in the definition of the Inverse Gap Weighting strategy in \eqref{eq:igw}.}

Let us now explain the difference in the two approaches and the reason for the differing rates. The first method uses a regression oracle that can predict rewards for sets of size $k$ (see Assumption~\ref{assm:oracle} below). Since the algorithm randomizes around a local optimum of the estimated function, it cannot distinguish good local optima (those with a higher approximation ratio of $1-e^{-1}$) from those with a $1/2$ approximation ratio. The second algorithm, on the other hand, explores subsets of a local optimum and thus obtains more information about the quality of the set. To be more precise, we follow  \cite{filmus2014monotone,sviridenko2017optimal} and define a surrogate potential function which eliminates the worse local optima. Computing this surrogate, however, requires a stronger oracle (Assumption~\ref{assm:oracle2}) that can estimate rewards for sets of size less than or equal to $k$. Furthermore, this extra exploration step comes at a cost of a worse rate of $n^{2/3}$.

\subsection{\textsf{SquareCB} for submodular optimization}

Following \cite{foster2020beyond}, we assume availability of an online regression oracle. In Section~\ref{sec:examples} we discuss several models for which computationally efficient oracles can be derived.
\begin{assumption}
\label{assm:oracle}
    There is an online regression algorithm for sequentially choosing $\widehat{r}_1, \ldots,\widehat{r}_n$ with $\widehat{r}_t=\widehat{r}_t(x_1,\S_1,r_1,\ldots,x_{t-1}, \S_{t-1}, r_{t-1}, x_t, \S_t)$ such that for any adaptively chosen sequence  $(x_1,\S_1),\ldots,(x_n,\S_n) \in \cX\times 2^{[A]}$ with $|\S_t|=k$ and outcomes $r_1,\ldots,r_n \in [0,1]$,
    \begin{align}
        \sum_{t=1}^n (\widehat{r}_t-r_t)^2 \leq \min_{\util \in \cF} \sum_{t=1}^n (\util(\S_t, x_t) - r_t)^2 + \Regsq.
        \label{eq:assm:oracle1}
    \end{align}
    Alternatively, a sufficient (weaker) assumption is that for any adaptively chosen sequence  $(x_1,\S_1),\ldots,(x_n,\S_n)$,
    \begin{align}
        \sum_{t=1}^n (\widehat{r}_t-\util^*(\S_t,x_t))^2 \leq \Regsq.
        \label{eq:assm:oracle11}
    \end{align}
\end{assumption}
We assume that the online oracle can be queried for the value of $\widehat{r}_t$ multiple times for various values of $(\S,x_t)$, and we will denote these values by $\widehat{\util}_t(\S,x_t)$.

Given a vector of scores $\pred\in [0,1]^K$ and a parameter $\gamma>0$, the Inverse Gap Weighting  distribution $p=\igw_{\gamma}(\pred)$ over $\{1,\ldots,K\}$ is defined as follows. Let $b=\arg\max_{i\in[K]} \pred(i)$. Define 
\begin{align}
\label{eq:igw}
    p(a) = \frac{1}{2K + \gamma (\pred(b)-\pred(a))}, ~~~ a \neq b
\end{align}
and $p(b) = 1-\sum_{a\neq b} p(a)$. This definition differs only slightly from the $\igw$ distribution considered in \cite{foster2020beyond}, which uses  $K$ rather than $2K$ in the denominator. The definition in \eqref{eq:igw} ensures that the  probability placed on the action with the highest score is at least $1/2$. Tracing through the proof in \cite[Lemma 3]{foster2020beyond}, we note that the strategy in \eqref{eq:igw} enjoys the same upper bound of $2K/\gamma$ for \eqref{eq:igw} as the one stated in \cite{foster2020beyond}: 
\begin{lemma}[adapted from \cite{foster2020beyond}]
\label{lem:igw}
    For any vector $\pred\in [0,1]^K$, the distribution $p=\igw_{\gamma}(\pred)$ in \eqref{eq:igw} ensures that for any $f^*\in[0,1]^K$,
\begin{align}
	\label{eq:igw_per_round}
	\sum_{a\in [K]} p(a) \left( \max_{a^*\in [K]} f^*(a^*)-f^*(a)  - \frac{\gamma}{4} (\pred(a)-f^*(a))^2 \right)\leq \frac{2K}{\gamma}
\end{align}
for $K\geq 2$.
\end{lemma}

\begin{algorithm}
\caption{\algoname}\label{alg}
\textbf{input:}  Parameter $\gamma>0$.
% \textbf{algorithm:}
\begin{algorithmic}[0]
    \For{$t=1,\ldots,n$}
    \State Receive $x_t$ and matroid $\cM_t=([A], \cI_t)$
    \State Find a local optimum $\Shat_t$ of $\utilhat_t(\cdot, x_t)$ with respect to $\cM_t$, with $|\Shat_t|=k$
    \State Compute $\pred_t=\{\utilhat_t\left(\S, x_t\right): \S\in \cI_t(\Shat_t)\}$
    %for $a\in\S_t, b\in\S_t^c$, and $\pred_{0} = \utilhat_t(\S_t,x_t)$.
    \State Sample $\S_t \sim p_t = \igw_{\gamma}(\pred_t)$
    %\State Let $\Shat_t = \S_t \setminus\{a\} \cup\{b\}$ if $\alpha=(a,b)$ or $\Shat_t=\S_t$ if $\alpha=0$.
    \State Obtain reward $r_t$ for the selected set $\S_t$
    \State Feed $(x_t,\S_t)$ as context and $r_t$ as target for Online Regression. Obtain updated model $\utilhat_{t+1}$
    \EndFor
\end{algorithmic}
\end{algorithm}

The algorithm \algoname~ is based on \textsf{SquareCB}. On round $t$, it first finds a local optimum $\Shat_t$ of $\utilhat_t(\cdot,x_t)$ that respects the matroid constraint $\Shat_t\in\cI_t$. The set of actions is then defined as the set $\cI_t(\Shat_t)$. While this is an algorithm-dependent set of allowed actions,  the $\igw$ distribution satisfies \eqref{eq:igw_per_round} for any set of actions that may depend on the context $x_t$ and the estimated function $\utilhat_t$.

\begin{proposition}
    \label{prop:main}
    Under the oracle Assumption~\ref{assm:oracle}, with high probability, the algorithm \algoname~  with $\gamma \propto \sqrt{nk(A-k)/\Regsq}$ has regret 
    \begin{align}
     \sum_{t=1}^n \frac{1}{2} \max_{\S_t^*\in \cI_t}\util^*(\S_t^*, x_t)  - \util^*(\S_t, x_t)    
    %&= (k+1)\left[ \frac{1}{2} \gamma \mathsf{Reg_{sq}} + \frac{2n(k\cdot(A-k)+1)}{\gamma} + 16 \log(2/\delta) \right]\\
    %&+ \sqrt{2n \log (2/\delta)} \\
    &= \mathcal{O}\left( k\sqrt{k (A-k)\cdot n \Regsq}\right)
\end{align}
    for any sequence of contexts $x_1,\ldots,x_n$ and matroids $\cM_1,\ldots,\cM_n$, even if chosen adaptively.
\end{proposition}
\begin{proof}[Proof of Proposition~\ref{prop:main}]
We let $\cH_{t}$ be the $\sigma$-algebra corresponding to  $\S_1,r_1,\ldots, r_t,\S_t.$ Note that $\Shat_t$ is a function of $\cH_{t-1}, x_t, \cM_t,$ and for simplicity of exposition we assume that $x_t,\cM_t$ are $\cH_{t-1}$-measurable. Applying Lemma~\ref{lem:igw} to the $t$th step of \algoname, conditionally on $\cH_{t-1}$,
\begin{align}
    \sum_{\S\in\cI_t(\Shat_t)} p_t(\S) \left( \max_{\S_t^*\in \cI_t(\Shat_t)} \util^*(\S_t^*,x_t)-\util^*(\S, x_t)  - \frac{\gamma}{4} (\utilhat_t(\S,x_t)-\util^*(\S, x_t))^2 \right)\leq \frac{2|\cI_t(\Shat_t)|}{\gamma}
\end{align}
which we write as
\begin{align}
    \label{eq:e1}
    \En_{\S_t\sim p_t}\left[ 
        \max_{\S_t^*\in \cI_t(\Shat_t)} \util^*(\S_t^*,x_t)-\util^*(\S_t, x_t) \right]   
        \leq  
        \frac{\gamma}{4} \En_{\S_t\sim p_t} (\utilhat_t(\S_t,x_t)-\util^*(\S_t, x_t))^2 + \frac{2|\cI_t(\Shat_t)|}{\gamma}.
\end{align}
The left-hand side of this inequality is lower-bounded by
\begin{align}
    \label{eq:e2}
    \frac{1}{2}\left( \max_{\S_t^*\in \cI_t(\Shat_t)} \util^*(\S_t^*, x_t)-\util^*(\Shat_t,x_t) \right)  
\end{align}
since $\igw$ puts probability at least $1/2$ on the greedy action $\Shat_t$ and the difference is nonnegative (by definition, $\Shat_t\in \cI_t(\Shat_t)$). On the other hand, using martingale concentration  inequalities (see \cite[Lemma 2]{foster2020beyond}), with probability at least $1-\delta$,
\begin{align}
    \sum_{t=1}^n \En_{\S_t\sim p_t} \left(\utilhat_t(\S_t, x_t)-\util^*(\S_t, x_t) \right)^2 \leq 2 \left[ \sum_{t=1}^n (\utilhat_t(\S_t, x_t)-r_t)^2 - (\util^*(\S_t, x_t)-r_t)^2\right] + 16\log(2/\delta) \label{eq:e3}
\end{align}
and
\begin{align}
    \sum_{t=1}^n \max_{\S_t^*\in \cI_t(\Shat_t)} \util^*(\S_t^*,x_t)-\util^*(\S_t, x_t) 
    &\leq \sum_{t=1}^n \En_{\S_t\sim p_t}\left[ 
        \max_{\S_t^*\in \cI_t(\Shat_t)} \util^*(\S_t^*,x_t)-\util^*(\S_t, x_t) \right]  + \sqrt{2n \log (2/\delta)}.
        \label{eq:e4}
\end{align}
Let us denote the above event by $\mathcal{E}$. Under this event, in view of \eqref{eq:e1}, \eqref{eq:e2},  \eqref{eq:e3}, and by the oracle assumption,
\begin{align}
\label{eq:e5}
    \sum_{t=1}^n \frac{1}{2}\left( \max_{\S_t^*\in \cI_t(\Shat_t)} \util^*(\S_t^*,x_t)-\util^*(\Shat_t, x_t) \right)  \leq \frac{\gamma}{2} \Regsq + \sum_{t=1}^n  \frac{2|\cI_t(\Shat_t)|}{\gamma} +  16\log(2/\delta).
\end{align}
We now overbound $|\cI_t(\Shat_t)|\leq k\cdot(A-k)+1$ and conclude that under the event $\mathcal{E}$,
\begin{align}
    \sum_{t=1}^n \max_{\S_t^*\in \cI_t(\Shat_t)}\util^*(\S_t^*, x_t)-\util^*(\Shat_t, x_t)  \leq n\varepsilon 
\end{align}
for 
$$\varepsilon = \frac{\gamma}{n} \Regsq + \frac{4(k\cdot(A-k)+1)}{\gamma} +  \frac{32\log(2/\delta)}{n}.$$
Hence, the tuple $(\Shat_1,\ldots,\Shat_n)$ is a joint local $\varepsilon$-optimum, in the sense of \eqref{eq:def:joint_opt}, with respect to the sum of $\util^*(\cdot, x_t)$ for $t=1,\ldots,n$. By Lemma~\ref{lem:local_opt}, under the event $\mathcal{E}$,
\begin{align}
    \frac{1}{2}\left(\sum_{t=1}^n \max_{\S_t^*\in \cI_t}\util^*(\S_t^*, x_t)   - nk\varepsilon\right) \leq \sum_{t=1}^n \max_{\S_t^*\in \cI_t(\Shat_t)}\util^*(\S_t^*, x_t) . 
\end{align}
Combining with \eqref{eq:e4}, under the event $\mathcal{E}$,
\begin{align*}
     \sum_{t=1}^n \frac{1}{2} \max_{\S_t^*\in \cI_t}\util^*(\S_t^*, x_t)  - \util^*(\S_t, x_t)    
    &\leq \sum_{t=1}^n \En_{\S_t\sim p_t}\left[ 
        \max_{\S_t^*\in \cI_t(\Shat_t)} \util^*(\S_t^*,x_t)-\util^*(\S_t, x_t) \right]  + \sqrt{2n \log (2/\delta)} + \frac{1}{2}nk\varepsilon .
\end{align*}
Hence, once again using \eqref{eq:e1}, \eqref{eq:e3}, and the oracle assumption, with probability at least $1-\delta$,
\begin{align}
     \sum_{t=1}^n \frac{1}{2} \max_{\S_t^*\in \cI_t}\util^*(\S_t^*, x_t)  - \util^*(\S_t, x_t)    
    &\leq \frac{1}{2}n(k+1)\varepsilon  + \sqrt{2n \log (2/\delta)} \\
    &= (k+1)\left[ \frac{1}{2} \gamma \Regsq + \frac{2n(k\cdot(A-k)+1)}{\gamma} + 16 \log(2/\delta) \right]\\
    &+ \sqrt{2n \log (2/\delta)} \\
    &= \mathcal{O}\left( k\sqrt{n k (A-k) \Regsq}\right)
\end{align}
with the choice of $\gamma \propto \sqrt{nk(A-k)/\Regsq}$.

\end{proof}

\begin{remark}
We now compare the above result to the linear case (equivalently, modular functions with $\util(\emptyset)=0$).
Let $\util_1,\ldots,\util_n$ be a sequence of linear functions of the form $\util_j(\S) = \sum_{a\in \S} f_j(a)$. The joint local $\varepsilon$-optimality in \eqref{eq:def:joint_opt} means
\begin{align}
    \sum_{j=1}^n \max_{a\in\S_j, b\in\S_j^c} [f_j(b)-f_j(a)]_+ = \sum_{j=1}^n \max_{\T_j \in \cI_j(\S_j)} \util_j (\T_j) - \sum_{j=1}^n \util_j(\S_j) \leq n\varepsilon,
\end{align}
where the operator $[x]_+ = \max\{x,0\}$ is introduced to include the case where the maximum is achieved at $\T_j=\S_j$. This, in turn, implies
\begin{align}
    \sum_{j=1}^n \max_{\T_j} \util_j(\T_j) - \util_j(\S_j) = \sum_{j=1}^n \max_{\T_j} \sum_{a\in \T_j} f_j(a) - \sum_{j=1}^n\sum_{b\in \S_j} f_j(b) \leq nk\varepsilon.
\end{align}
It follows that Lemma~\ref{lem:local_opt} holds with $1/2$ replaced with $1$ in \eqref{eq:optimality_of_loc} when the functions are linear. This improvement implies a standard definition of regret with $c=1$ in \eqref{eq:regret_def} and a slightly improved constant in the quantitative version of the regret bound in Proposition~\ref{prop:main} (see the proof).
\end{remark}

\subsection{Improving the multiplicative constant}

For the improved multiplicative constant  of $1-e^{-1}$, we require a stronger online regression oracle that can learn the values for subsets of size of $k$ or smaller. Using this information, good local optima can be distinguished from the worse local optima.

\begin{assumption}
\label{assm:oracle2}
There is an online regression algorithm that satisfies Assumption~\ref{assm:oracle} with $|\S_t|=k$ replaced by $|\S_t|\leq k$.    
\end{assumption}

We also assume that each matroid $\cM_t$ has, for simplicity of exposition, the same rank $k$.

\begin{algorithm}
\caption{$\epsilon$-Greedy}\label{alg2}
\textbf{input:}  Parameter $\rho\in (0,1/2)$.
% \textbf{algorithm:}
\begin{algorithmic}[0]
    \For{$t=1,\ldots,n$}
    \State Receive $x_t$ and matroid $\cM_t=([A], \cI_t)$
    \State Find a local optimum $\Shat_t\in\cB(\cM_t)$ of $\cT\utilhat_t(\cdot, x_t)$
    \State Sample $\S_t'\sim \text{unif}(\cB(\Shat_t))$
    \State Sample $\S_t \sim (1-\rho) \delta_{\Shat_t} + \rho \cD_{\S_t'} $
    \State Obtain reward $r_t$ for the selected set $\S_t$
    \State Feed $(x_t,\S_t)$ as context and $r_t$ as target for Online Regression. Obtain updated model $\utilhat_{t+1}$
    \EndFor
\end{algorithmic}
\end{algorithm}

Rather than directly finding a local optimum of $\utilhat_t$, Algorithm~\ref{alg2} finds a local optimum of $\cT\utilhat_t$ on round $t$. Next, rather than sampling according to the Inverse Gap Weighting strategy, Algorithm~\ref{alg2} samples uniformly in the local neighborhood of $\Shat_t$. The set $\S_t$ is then equal to $\Shat_t$ with probability $1-\rho$, and with the remaining probability $\rho$ a subset of $\S_t'$ is chosen according to the distribution $\cD_{\S_t'}$ defined in \eqref{def:dist_subsets}. The latter step ensures that information about the surrogate potential $\cT\utilhat_t$ is collected.

\begin{proposition}
    \label{prop:main_T}
    Under Assumption~\ref{assm:oracle2}, Algorithm~\ref{alg2} has, with high probability, regret 
    \begin{align}
     \sum_{t=1}^n (1-e^{-1}) \max_{\S_t^*\in \cI_t}\util^*(\S_t^*, x_t)  - \util^*(\S_t, x_t)    
    %&= (k+1)\left[ \frac{1}{2} \gamma \mathsf{Reg_{sq}} + \frac{2n(k\cdot(A-k)+1)}{\gamma} + 16 \log(2/\delta) \right]\\
    %&+ \sqrt{2n \log (2/\delta)} \\
    &= \mathcal{O}\left( ((A-k) k^2 \log k )^{2/3} \cdot n^{2/3} (\Regsq)^{1/3} \right)
\end{align}
    for any sequence of contexts $x_1,\ldots,x_n$ and matroids $\cM_1,\ldots,\cM_n$, even if chosen adaptively.
\end{proposition}

\section{Examples}
\label{sec:examples}

In this section, our goal is to provide a few simple examples of online regression oracles that could be used in practice together with the algorithms from previous sections. We start by noting that for any finite collection $\cF$, one can simply use Vovk's aggregating forecaster (see e.g. \cite[Section 3.5]{cesa2006prediction}) and obtain $\Regsq = \mathcal{O}(\log |\cF|)$ (see \cite{foster2020beyond} for further discussion). Looking back at the example in the introduction, such an aggregating forecaster would be applicable if in \eqref{eq:example_ranking} we choose $\lambda_i(x)$ and $f_i(\T,x)$ from a finite collection. The techniques developed in the online learning literature, however, allow one to go beyond finite classes. 

\subsection{Measuring diversity of a set}

We start by proposing a nonnegative monotone submodular function for a set of elements represented by $d$-dimensional vectors. With a slight abuse of notation, we identify elements $[A]$ with their vector representations. Given a set $\S\subset \reals^d$ of vectors, the measure $F(\S)$ should indicate whether the vectors span a sufficiently diverse set of directions. A possible approach that received attention in the machine learning community is the submodular function
$$F(\S) = \log \text{det} (L_\S)$$
where $L$ is a positive definite matrix with rows indexed by elements of the ground set and $L_\S$ denoting the $|\S|\times |\S|$ submatrix corresponding to $\S$. As a concrete example, one may take $L$ to be the Gram matrix $L_{i,j} = \inner{s_i,s_j}$ for $s_i,s_j\in\S$. In this case, the determinant can be viewed as measuring the volume of the set. Unfortunately, this function is not monotone in general. To address this shortcoming, we show that the notions of Gaussian width or Rademacher averages of the set are submodular and monotone. Furthermore, thanks to standard concentration-of-measure arguments, they are easy to estimate to accuracy $\epsilon$ in time $\mathcal{O}(d|\S|\epsilon^{-2})$.
\begin{lemma}
	\label{lem:submodular_width}
 	Let $\S\subset \reals^d$. Let $\eta$ be a random variable in $\reals^d$ with $\En\eta= 0$. Then 
	\begin{align}
		\label{eq:def_gaussian_width}
		\cW(\S) = \En \max_{s\in \S} \inner{s, \eta}
	\end{align}
	is monotone submodular and nonnegative. Furthermore, if, for instance, $\eta_1,\ldots,\eta_k \sim \mathcal{N}(0,I_d)$ independent, 
	$$\left|\frac{1}{k}\sum_{j=1}^k \max_{s\in \S} \inner{s, \eta_j} - \cW(\S)\right| \leq \text{diam}(\S)\sqrt{\frac{2\log (2/\delta)}{k}}.$$
\end{lemma}

We can extend definition \eqref{eq:def_gaussian_width} in a number of ways. In particular, we can define
\begin{align}
	\cW (\S, \Sigma) = \En \max_{s\in \S} \inner{s, \Sigma \eta}
\end{align}
for a positive semidefinite matrix $\Sigma$. When thought of as part of a context, $\Sigma$ can emphasize important features by rescaling.

\subsection{A simple model}

In this section we specify the online regression oracle for the following model. Let $\vtil\in \cF_{\textsf{nms}}$, and fix a nondecreasing $1$-Lipschitz function $\sigma:\reals\to [0,1]$. Let
\begin{align}
	\label{def_class_simple_submodular}
	\cF = \left\{ (\S,x) \mapsto \vtil(\S, x) \sigma(\inner{\theta,x}) : \norm{\theta}_2\leq 1 \right\}.
\end{align}
Functions in this class are parametrized by a vector $\theta\in\reals^d = \cX$. While this class is relatively simple, the square loss 
$$\theta \mapsto (\vtil(\S, x) \sigma(\inner{\theta,x}) - r)^2$$
is not necessarily convex in the parameter $\theta$. However, a computationally-efficient GLMtron-style algorithm guarantees small regret in the well-specified setting. The proof of the following result is essentially from \cite{foster2020beyond}, while the technique for learning GLMs with a `modified gradient' is from \cite{kalai2009isotron,kakade2011efficient}.
\begin{algorithm}
\caption{GLMtron-style algorithm}\label{alg:isotron-simple}
% \textbf{input:}  \\
% \textbf{algorithm:}
\begin{algorithmic}[0]
	\State Input: learning rate $\eta$. Initialize: $\theta_1=0$.
    \For{$t=1,\ldots,n$}
    \State Observe $\S_t, x_t$. 
    \State Predict $\widehat{r}_t = \vtil(\S_t,x_t)\sigma(\inner{\theta_t,x_t})$.
	\State Observe outcome $r_t$ and update the model 
		$$\theta_{t+1} = \text{Proj}_{\mathsf{B}_2(1)} \left(\theta_t - \eta g_t\right)~~~~\text{where}~~~~ %$\text{Proj}$ is projection onto unit Euclidean ball and
	g_t = \vtil(\S_t,x_t)(\vtil(\S_t,x_t)\sigma(\inner{\theta_t,x_t})-r_t)x_t.$$
    \EndFor
\end{algorithmic}
\end{algorithm}
\begin{proposition}
	\label{prop:isotron_single}
	Assume that $r_t\in [0,1]$, $\vtil(\S_t,x_t)\leq 1$, and $\norm{x_t}\leq 1$ almost surely. With the setting of $\eta=n^{-1/2}$, the regret of Algorithm~\ref{alg:isotron-simple} is bounded, with probability at least $1-\delta$, as
	\begin{align}
		\sum_{t=1}^n (\widehat{r}_t-\util^*(\S_t,x_t))^2 \leq \mathcal{O}(\sqrt{n\log(1/\delta)})
	\end{align}
	for any sequence $(x_1, \S_1),\ldots, (x_n, \S_n)$ and $\En[r_t|\S_t,x_t] = \util^*(\S_t,x_t)$ for $\util^*\in\cF$.
\end{proposition}

While for brevity we used in \eqref{def_class_simple_submodular} the same context $x$ in the submodular part $\vtil(\S,x)$ and in $\sigma(\inner{\theta,x})$, it should be clear that the results subsume, for instance, the model $\cW(\S,\Sigma)\sigma(\inner{\theta,x})$. In the next section, we propose a more complex extension.

\subsection{Sum of GLMs}

We would like to extend the model \eqref{def_class_simple_submodular} to a sum of several submodular functions weighted by monotonically increasing functions, as to cover the example in the Introduction. To be precise, let $\cX \subseteq \reals^d$ and fix $k$ matrices $P_1,\ldots,P_k \in \reals^{d\times d}$. Fix functions $\vtil_1,\ldots,\vtil_k\in \cF_{\textsf{nms}}$ and define
\begin{align}
	\label{def_class_sum_submodular}
	\cF = \left\{ (\S,x) \mapsto \sum_{i=1}^k \vtil_i(\S,x) \sigma(\theta^\tr P_i x) : \norm{\theta} \leq 1 \right\}.
\end{align}
For instance, for product recommendation examples, we may posit a model of the form
	\begin{align}
		\util(\S,x) = \sum_{\C\in\cC} \cW(\S\cap \C) \sigma(\inner{\theta_{\C}, x})
	\end{align}
	where $\cC$ is a collection of categories (specified by sets $\C$). In this case, $\cW(\S\cap \C)$ measures the quality of diverse coverage of each category $\C$, weighted by the user-specific function $\sigma(\inner{\theta_{\C}, x})$ that measures general interest in the category. The model is a natural `contextualized' generalization of the model in \cite{yue2011linear, yue2008predicting}, and can be viewed as an instance of \eqref{def_class_sum_submodular} with block-diagonal indicators $P_i$.
	
While it is not difficult to certify existence of an online regression algorithm for the class in \eqref{def_class_sum_submodular} (using the techniques in \cite{rakhlin2015online,rakhlin2014online}), a general form of a computationally efficient algorithm in the spirit of Algorithm~\ref{alg:isotron-simple} is unclear. Instead, we proceed as in \cite{goel2018learning} and make a further simplifying assumption that $\sigma(a)=\max\{a,0\}$ is a ReLU function. In addition, we assume that $x_1,\ldots,x_n$ are i.i.d. draws from a symmetric distribution with $\Sigma = \En x_t x_t^\tr$, and assume a uniform lower bound on the eigenvalues of the following matrices:
	\begin{align}
		\label{assmpt:least_eval}
		\forall t=1,\ldots,n,~~~~ \En_{x} \left[ \sum_{i,j=1}^k  \vtil_i(\S_t,x)\vtil_j(\S_t,x) P_i x x^\tr  P_j^\tr \right] \succeq \lambda_{\textsf{min}} I, ~~~~\lambda_{\textsf{min}} >0
	\end{align}
	almost surely. The condition ensures that there is some nonzero strength of the signal being conveyed to the algorithm. 

\begin{algorithm}
\caption{GLMtron-style algorithm}\label{alg:isotron-simple2}
\begin{algorithmic}[0]
	\State Input: learning rate $\eta$. Initialize: $\theta_1=0$.
    \For{$t=1,\ldots,n$}
    \State Observe $\S_t, x_t$. 
    \State Predict $\widehat{r}_t = \sum_{i=1}^k \vtil_i(\S_t,x_t)\sigma(\theta_t^\tr P_i x_t)$.
	\State Observe outcome $r_t$ and update the model 
		$$\theta_{t+1} = \text{Proj}_{\mathsf{B}_2(1)} \left(\theta_t - \eta g_t\right)~~~~\text{where}~~~~ %$\text{Proj}$ is projection onto unit Euclidean ball and
	g_t = \left( \sum_{i=1}^k \vtil_i(\S_t,x_t)\sigma(\theta_t^\tr P_i x_t) - r_t \right) \left( \sum_{j=1}^k \vtil_j(\S_t,x_t) P_j \right) x_t .$$
    \EndFor
\end{algorithmic}
\end{algorithm}

The proof of the following Proposition is similar to the one in \cite{goel2018learning}. 
\begin{proposition}
	\label{prop:multi_glm}
	Let $\sigma(a)=\max\{a,0\}$. Suppose $\vtil_i$ is symmetric: $\vtil_i(\S_t,x)=\vtil_i(\S_t,-x)$, for all $i=1,\ldots,k$. Suppose condition \eqref{assmpt:least_eval} holds for some $\lambda_{\textsf{min}}>0$, and let $\lambda_{\textsf{max}}$ denote the largest eigenvalue of $\sum_{i=1}^k P_i \Sigma P_i^\tr$. Assume that $r_t\in [0,1]$, $\vtil_i(\S_t,x_t)\leq 1$, and $\norm{x_t}\leq 1$ almost surely. With the setting of $\eta \propto n^{-1/2}$, the regret of Algorithm~\ref{alg:isotron-simple} is bounded, with probability at least $1-\delta$, as
	\begin{align}
		\sum_{t=1}^n (\widehat{r}_t-\util^*(\S_t,x_t))^2 \leq \mathcal{O}(\sqrt{n\log(1/\delta)})
	\end{align}
	for any sequence $\S_1,\ldots,\S_n$, an i.i.d. sequence of $x_1,\ldots,x_n$ with symmetric distribution, and the well-specified model $\En[r_t|\S_t,x_t] = \util^*(\S_t,x_t)$ for $\util^*\in\cF$ defined in \eqref{def_class_sum_submodular}. The $\mathcal{O}$ notation hides a factor $\frac{k \lambda_{\textsf{max}}}{\lambda_{\textsf{min}}} $.
\end{proposition}

\appendix

\section{Proofs}

\begin{proof}[Proof of Lemma~\ref{lem:local_opt}]
The left-hand side of \eqref{eq:optimality_of_loc} is immediate since $\S_j\in\cI_j(\S_j)$. We now prove the right-hand side. Let $\S_j=\{s_1^j,\ldots,s_k^j\} \in \cI_j$, $\T_j=\{t_1^j,\ldots,t_k^j\} \in \cI_j$. By \cite[Corollary 39.12a]{schrijver2003combinatorial}, there is a bijection $\pi_j$ between $\S_j\setminus \T_j$ and $\T_j\setminus\S_j$ such that $\S_j\setminus\{t_i^j\}\cup \{\pi_j(t_i^j)\} \in \cI_j$ for any $i\in[k]$. We extend the bijection to be identity on $\T_j\cap\S_j$. Then
\begin{align}
    \sum_{j=1}^n\util_j(\S_j\cup\T_j) - \util_j(\S_j) &= \sum_{j=1}^n\sum_{i=1}^k \util_j(\S_j\cup\{t_1^j,\ldots,t_i^j\}) - \util_j(\S_j\cup\{t_1^j,\ldots,t_{i-1}^j\}) \label{eq1}\\
    & \leq \sum_{j=1}^n\sum_{i=1}^k \util_j(\S_j\cup\{t_i^j\}) - \util_j(\S_j) \label{eq2}\\
    &\leq \sum_{j=1}^n \sum_{i=1}^k \util_j(\S_j\setminus\{\pi_j(t_i^j)\}\cup\{t_i^j\}) - \util_j(\S_j\setminus\{\pi_j(t_i^j)\}) \label{eq3}\\
    &= \sum_{i=1}^k \sum_{j=1}^n  \util_j(\S_j\setminus\{\pi_j(t_i^j)\}\cup\{t_i^j\}) - \util_j(\S_j\setminus\{\pi_j(t_i^j)\}) \label{eq3a}\\
    &\leq \sum_{i=1}^k \left[\sum_{j=1}^n\util_j(\S_j) - \util_j(\S_j\setminus\{\pi_j(t_i^j)\}) + n\varepsilon \right]\label{eq4}\\
    %&\leq |\S^*| \cdot (\widehat{\Delta}+\varepsilon)
    &= \sum_{j=1}^n\sum_{i=1}^k \util_j(\S_j) - \util_j(\S_j\setminus\{s_i^j\}) + nk\varepsilon \label{eq5}\\
    &\leq \sum_{j=1}^n \sum_{i=1}^k \util_j(\{s_1^j,\ldots,s_i^j\}) - \util_j(\{s_1^j,\ldots,s_{i-1}^j\}) + nk\varepsilon \label{eq6}\\
    &= \sum_{j=1}^n \util_j(S_j) + nk\varepsilon \label{eq7}
\end{align}
where \eqref{eq2} and \eqref{eq3} are due to submodularity, \eqref{eq4} is by  \eqref{eq:def:joint_opt}, \eqref{eq6} is by submodularity, and \eqref{eq7} is by telescoping. 
Hence, by monotonicity
    $$\sum_{j=1}^n \util_j(\T_j) \leq \sum_{j=1}^n \util_j(\T_j\cup \S_j) \leq 2\sum_{j=1}^n\util_j(\S_j)+ nk\varepsilon.$$
    
\end{proof}

\begin{lemma}[\cite{filmus2014monotone}]
    \label{lem:gaps_subm_lin}
    Let $\S=\{s_1,\ldots,s_k\}$ and $\T=\{t_1,\ldots,t_k\}$ be any two bases of a matroid $\cM$ and suppose that the elements of $\S$ are indexed according to bijection $\pi$ so that $\S-s_i+t_i \in \cB(\cM)$ for all $1\leq i\leq k$. Then
    \begin{align}
        \util(\S) \geq (1-e^{-1})\util(\T) + (1-e^{-1})\sum_{i=1}^k [ \cT\util(\S) - \cT\util(\S-s_i+t_i)]
    \end{align}
\end{lemma}

\begin{proof}[Proof of Lemma~\ref{lem:local_opt_T}]
    Once again, the left-hand side of \eqref{eq:optimality_of_loc_T} is immediate since $\S_j\in\cB_j(\S_j)$. We now prove the right-hand side. Let $\S_j=\{s_1^j,\ldots,s_k^j\} \in \cB_j$, $\T_j=\{t_1^j,\ldots,t_k^j\} \in \cB_j$, and let $\pi_j$ be a bijection $\pi_j$ as in the proof of Lemma~\ref{lem:local_opt}. Then, using the result of Lemma~\ref{lem:gaps_subm_lin},
    \begin{align}
    &  (1-e^{-1})\util_j(\T_j) - \util_j(\S_j)  \\
    &= (1-e^{-1})\util_j(\T_j) - \util_j(\S_j) + (1-e^{-1})\sum_{i=1}^k [ \cT \util_j(\S_j) - \cT \util_j(\S_j - \pi(t_i)+t_i) ] \\
    &- (1-e^{-1})\sum_{i=1}^k [ \cT \util_j(\S_j) - \cT \util_j(\S_j-\pi(t_i)+t_i) ] \\
    &\leq (1-e^{-1}) k\left(\max_i \cT \util_j(\S_j-\pi(t_i)+t_i) - \cT \util_j(\S_j) \right)\\
    &\leq (1-e^{-1})k \left(\max_{\T_j\in\cB_j(\S_j)} \cT \util_j(\T_j) - \cT \util_j(\S_j) \right)
    \end{align}
    Rearranging and summing over $j=1,\ldots,n$ completes the proof.
\end{proof}

\begin{proof}[Proof of Proposition~\ref{prop:main_T}]
To simplify the proof, we assume a uniform upper bound $B$ on the size $\cB(\Shat_t)$ of any local neighborhood of the set $\Shat_t$. Furthermore, note that $\tau(\Shat_t)$ only depends on the cardinality of $\Shat_t$, and thus we abuse the notation and write $\tau(k)$ for $\tau(\Shat_t)$.

Now, conditionally on $\cH_{t-1}, x_t, \cM_t$, for any $\S^*_t \in \cB(\cM_t)$,
\begin{align}
    &\En_{\S_t',\S_t} \left[ (1-e^{-1})\util^*(\S^*_t,x_t) - \util^*(\S_t,x_t) \right] \\
    &\leq (1-e^{-1})\util^*(\S^*_t,x_t) - \util^*(\Shat_t,x_t) + \En_{\S_t',\S_t} \left[ \util^*(\Shat_t, x_t) - \util^*(\S_t,x_t) \right] \\
    &\leq (1-e^{-1})\util^*(\S^*_t, x_t) - \util^*(\Shat_t, x_t) + \rho,
\end{align}
where the last step follows from the assumption that $\util^*$ takes values in $[0,1]$. On the other hand, by Lemma~\ref{lem:local_opt_T},
\begin{align}
    (1-e^{-1})\sum_{t=1}^n \util^*(\S^*_t, x_t) &\leq \sum_{t=1}^n \util^*(\Shat_t, x_t)  + (1-e^{-1})nk \varepsilon
\end{align}
whenever
\begin{align}
\label{eq:subopt11}
    \sum_{t=1}^n \max_{\S^*\in\cB(\Shat_t)} \cT \util^*(\S^*, x_t) - \cT \util^*(\Shat_t, x_t)  \leq n\varepsilon.
\end{align}
We now quantify $\varepsilon$ in \eqref{eq:subopt11}. It holds that 
\begin{align}
    &\max_{\S^*\in\cB(\Shat_t)} \cT \util^*(\S^*, x_t) - \cT \util^*(\Shat_t, x_t) \\
    &\leq \max_{\S^*\in\cB(\Shat_t)} \cT \util^*(\S^*, x_t) - \cT\utilhat_t(\S^*, x_t) + \cT\utilhat_t(\Shat_t, x_t)- \cT \util^*(\Shat_t, x_t) \\
    %&\leq \max_{\S^*\in\cB(\Shat_t)} |\cT \util^*(x_t,\S^*) - \cT\utilhat_t(x_t,\S^*) | + |\cT\utilhat_t(x_t,\Shat_t)- \cT \util^*(x_t,\Shat_t)| \\
    &\leq 2 \max_{\S^*\in\cB(\Shat_t)} |\cT \util^*(\S^*, x_t) - \cT\utilhat_t(\S^*, x_t) | 
\end{align}
since $\Shat_t$, by definition, is a local maximum of $\utilhat_t(x_t,\cdot)$ in $\cB(\Shat_t)$. The last expression is at most
\begin{align}
    &2
        B \cdot \En_{\S'_t\sim \text{unif}(\cB(\Shat_t))} |\cT \util^*(\S'_t, x_t) - \cT\utilhat_t(\S'_t, x_t) | \\
    &\leq 2
        B \tau(k)\cdot \En_{\S'_t\sim \text{unif}(\cB(\Shat_t)), \S\sim \cD_{\S'_t}} |\util^*(\S, x_t) - \utilhat_t(\S, x_t) |,
\end{align}
using the definition in \eqref{eq:def_cTvtil}.

Summing over $t=1,\ldots,n$, assuming $\rho\in(0,1/2)$, and defining the shorthand $\phi= 2(1-e^{-1})k B \tau(k)$,
\begin{align}
    &\sum_{t=1}^n \En_{\S_t',\S_t} \left[ (1-e^{-1})\util^*(\S^*_t, x_t) - \util^*(\S_t, x_t) \right] \\
    &\leq \rho n 
    + \frac{\phi}{\sqrt{\rho}}\left( \sum_{t=1}^n \En_{\S'_t\sim \text{unif}(\cB(\Shat_t)), \S\sim \cD_{\S'_t}} \sqrt{\rho}| \util^*(\S, x_t) - \utilhat_t(\S, x_t) | + \sqrt{1-\rho}|\utilhat_t(\Shat_t, x_t)- \util^*(\Shat_t, x_t)| \right) .
\end{align}
Furthermore,
\begin{align}
   &\sum_{t=1}^n \En_{\S'_t\sim \text{unif}(\cB(\Shat_t)), \S\sim \cD_{\S'_t}} \sqrt{\rho}| \util^*(\S, x_t) - \utilhat_t(\S, x_t) | + \sqrt{1-\rho}|\utilhat_t(\Shat_t, x_t)- \util^*(\Shat_t, x_t)| \\
   &\leq \sqrt{2n \sum_{t=1}^n \En_{\S'_t\sim \text{unif}(\cB(\Shat_t)), \S\sim \cD_{\S'_t}} \rho\left( \util^*(\S, x_t) - \utilhat_t(\S, x_t) \right)^2 + (1-\rho) (\utilhat_t(\Shat_t, x_t)- \util^*(\Shat_t, x_t))^2} \\
   &\leq \sqrt{2n \sum_{t=1}^n \En_{\S'_t\sim \text{unif}(\cB(\Shat_t)), \S_t \sim (1-\rho) \delta_{\Shat_t} + \rho \cD_{\S_t'} } \left( \util^*(\S_t, x_t) - \utilhat_t(\S_t, x_t) \right)^2 } .
\end{align}
As in \eqref{eq:e3}, with probability at least $1-\delta$,
\begin{align}
    \sum_{t=1}^n \En_{\S_t', \S_t} \left(\utilhat_t(\S_t, x_t)-\util^*(\S_t, x_t) \right)^2 
    &\leq 2 \left[ \sum_{t=1}^n (\utilhat_t(\S_t, x_t)-r_t)^2 - (\util^*(\S_t, x_t)-r_t)^2\right] + 16\log(2/\delta) \label{eq:e32} \\
    &\leq 2\Regsq + 16\log(2/\delta).
\end{align}
Putting everything together,
\begin{align}
    \sum_{t=1}^n \En_{\S_t',\S_t} \left[ (1-e^{-1})\util^*(\S^*_t, x_t) - \util^*(\S_t, x_t) \right] &\leq \rho n +  \frac{\phi}{\sqrt{\rho}}\sqrt{2n (2\Regsq + 16\log(2/\delta))} 
\end{align}
Setting
$$\rho = n^{-1/3} (2(1-e^{-1})k B \tau(k))^{2/3} (4\Regsq + 32\log(2/\delta))^{1/3}$$
balances the two terms and yields an upper bound of 
$$2n^{2/3} (2(1-e^{-1})k B \tau(k))^{2/3} (4\Regsq + 32\log(2/\delta))^{1/3}.$$
Martingale concentration results yield the final high-probability bound. Finally, as shown in \cite{filmus2014monotone}, $\tau(k) \leq \frac{e}{e-1}H_k \leq C\log k$, where $H_k$ is the $k$th harmonic number. Further, we overbound $B\leq  k\cdot(A-k)+1$, as in the proof of Proposition~\ref{prop:main}. The resulting dependence on $k$ and $A$ is then
$((A-k) k^2 \log k )^{2/3}.$
\end{proof}

\begin{proof}[Proof of Lemma~\ref{lem:submodular_width}]
	For every $\eta\in\reals^d$,
	\begin{align}
		\max_{s\in \S \cup \T} \inner{s, \eta}= \max\left\{ \max_{s\in \S} \inner{s, \eta}, \max_{s\in \T} \inner{s, \eta} \right\}~~~~\text{and}~~~~
		\max_{s\in \S \cap \T} \inner{s, \eta} \leq \min\left\{ \max_{s\in \S} \inner{s, \eta}, \max_{s\in \T} \inner{s, \eta} \right\}.
	\end{align}
	Hence,
	\begin{align}
		\max_{s\in \S \cup \T} \inner{s, \eta} + \max_{s\in \S \cap \T} \inner{s, \eta} \leq \max_{s\in \S} \inner{s, \eta} + \max_{s\in \T} \inner{s, \eta}.
	\end{align}
	Taking expectation on both sides establishes the submodularity claim. Monotonicity is trivial while nonnegativity follows by Jensen's inequality and the assumption of $\En \eta =0$. 
	Finally, $\boldsymbol{\eta}=(\eta_1,\ldots,\eta_k) \mapsto \frac{1}{k}\sum_{j=1}^k \max_{s\in \S} \inner{s, \eta_j}$ is Lipschitz:
	$$\frac{1}{k}\sum_{j=1}^k\max_{s\in \S} \inner{s, \eta_j}- \frac{1}{k}\sum_{j=1}^k\max_{s\in \S} \inner{s, \eta'_j} \leq \frac{1}{k}\sum_{j=1}^k \max_{s\in \S} \inner{s, \eta_j-\eta'_j} \leq \text{diam}(\S) \frac{1}{k}\sum_{j=1}^k \norm{\eta_j-\eta'_j} \leq \frac{\text{diam}(\S)}{\sqrt{k}}\norm{\boldsymbol{\eta}-\boldsymbol{\eta'}}.$$
	Hence, Gaussian concentration yields 
	\begin{align}
		\mathbb{P}\left(\left|\frac{1}{k}\sum_{j=1}^k\max_{s\in\S} \inner{s,\eta_j} - \cW(\S)\right| \geq u\right) \leq 2\exp\left\{-\frac{k u^2}{2\text{diam}(\S)^2}\right\}.
	\end{align}	
\end{proof}

\begin{proof}[Proof of Proposition~\ref{prop:isotron_single}]
	Let $\util^*\in\cF$ be parametrized by $\theta^*\in \reals^d$ with $\norm{\theta^*}\leq 1$. First, note that
	\begin{align}
		\En_{t-1} g_t &= \vtil(\S_t,x_t)(\vtil(\S_t,x_t)\sigma(\inner{\theta_t,x_t})-\En [r_t|\S_t,x_t] )x_t \\
		&=\vtil(\S_t,x_t)^2\left( \sigma(\inner{\theta_t,x_t})-\sigma(\inner{\theta^*,x_t})\right) x_t. 
	\end{align}

	We have
	\begin{align}
		(\widehat{r}_t - \util^*(\S_t,x_t))^2 &= \left( \vtil(\S_t, x_t)\sigma(\inner{\theta_t,x_t}) - \vtil(\S_t, x_t)\sigma(\inner{\theta^*,x_t})\right)^2\\
		&= \vtil(\S_t, x_t)^2 \left( \sigma(\inner{\theta_t,x_t}) - \sigma(\inner{\theta^*,x_t})\right)^2\\
		%&\leq \left( \vtil(\S_t, x_t)\right) \cdot \vtil(\S_t, x_t)\left(\sigma_i(\inner{\theta^t_i,x_t}) - \sigma_i(\inner{\theta^*_i,x_t})\right)^2 \\
		&\leq \vtil(\S_t, x_t)^2 \cdot \left(\sigma(\inner{\theta_t,x_t}) - \sigma(\inner{\theta^*,x_t})\right)(\inner{\theta_t,x_t}-\inner{\theta^*,x_t}) \\
		&= \En_{t-1} \inner{g_t, \theta_t-\theta^*} 
	\end{align}
	By Azuma-Hoeffding inequality, with probability at least $1-\delta$,
	\begin{align}
		\sum_{t=1}^n \En_{t-1} \inner{g_t, \theta_t-\theta^*}  \leq \sum_{t=1}^n  \inner{g_t, \theta_t-\theta^*} + G\sqrt{c n \log (1/\delta)}
	\end{align}
	where $c$ is an absolute constant and $G$ is an almost sure bound on $\norm{g_t}$. Standard analysis of online linear optimization yields 
	$$\sum_{t=1}^n  \inner{g_t, \theta_t-\theta^*}\leq G\sqrt{n}$$
	by rearranging and summing the terms in
	$$\norm{\theta_{t+1}-\theta^*}^2 \leq \norm{\theta_{t}-\eta g_t - \theta^*}^2 = \norm{\theta_{t} - \theta^*}^2 - 2\eta\inner{g_t, \theta_t-\theta^*}+ \eta^2 G^2.$$
\end{proof}

\begin{proof}[Proof of Proposition~\ref{prop:multi_glm}]
For brevity, let us write $\vtil_{i,t} = \vtil_i(\S_t,x_t)$. Observe that 
	\begin{align}
			\En_{x_t}\left(\sum_{i=1}^k \vtil_{i,t}\sigma(\theta^\tr_t P_i x_t) - \vtil_{i,t}\sigma(\theta_*^\tr P_i x_t) \right)^2 
			&\leq k \En_{x_t} \sum_{i=1}^k \vtil_{i,t}^2 \left(\sigma(\theta^\tr_t P_i x_t) - \sigma(\theta_*^\tr P_i x_t) \right)^2 \\
			&\leq k \En_{x_t}\sum_{i=1}^k \vtil_{i,t}^2 \left((\theta_t-\theta_*)^\tr P_i x_t\right)^2 \\
			&= k \sum_{i=1}^k (\theta_t-\theta_*)^\tr P_i \Sigma P_i^\tr (\theta_t-\theta^*) \\
			&\leq k \lambda_{\textsf{max}} \norm{\theta_t-\theta_*}^2.
	\end{align}
	Next,
	\begin{align}
		\En_{r_t} g_t &= \En_{r_t} \left[ \left( \sum_{i=1}^k \vtil_{i,t}\sigma(\theta_t^\tr P_i x_t) - r_t \right) \left( \sum_{j=1}^k \vtil_{j,t} P_j \right) x_t \right] \\
		& = \left[ \left( \sum_{i=1}^k \vtil_{i,t}(\sigma(\theta_t^\tr P_i x_t) - \sigma(\theta_*^\tr P_i x_t)) \right) \left( \sum_{j=1}^k \vtil_{j,t} P_j \right) x_t \right]
	\end{align}
	and
	\begin{align}
		\label{eq:grad_eq1}
		\En_{x_t, r_t}\inner{g_t, \theta_t-\theta_*} &= \En_{x_t} \inner{  \sum_{i,j=1}^k \vtil_{i,t}(\sigma(\theta_t^\tr P_i x_t) - \sigma(\theta_*^\tr P_i x_t))   \vtil_{j,t} P_j x_t, \theta_t - \theta_*} \notag\\
		&= \En_{x_t}\sum_{i,j = 1}^k \vtil_{i,t}(\sigma(\theta_t^\tr P_i x_t) - \sigma(\theta_*^\tr P_i x_t)) (\vtil_{j,t}\theta_t^\tr P_j x_t - \vtil_{j,t}\theta_*^\tr P_j x_t). 
	\end{align}
	We now use the trivial property that for $\sigma(a)=\max\{a,0\}$ and $x$ with a symmetric distribution, for all $v, w \in \reals^d$, 
	$$\En_{x} \left[ \sigma(v^\tr x)~ w^\tr x \right] = \frac{1}{2}\En_{x} \left[ v^\tr x~ w^\tr x \right].$$
	With this, the expression in \eqref{eq:grad_eq1} is equal to 
	\begin{align}
		&\frac{1}{2} \En_{x_t} \sum_{i,j=1}^k \vtil_{i,t}\vtil_{j,t}(\theta_t^\tr P_i x_t - \theta_*^\tr P_i x_t) (\theta_t^\tr P_j x_t - \theta_*^\tr P_j x_t) \\
		&= \frac{1}{2}  (\theta_t - \theta_*)^\tr \left[ \En_{x_t} \sum_{i,j=1}^k  \vtil_{i,t}\vtil_{j,t} P_i x_t x_t^\tr  P_j^\tr \right] (\theta_t - \theta_*) \\
		&\geq \norm{\theta_t - \theta_*}^2 \cdot \lambda_{\textsf{min}}/2
	\end{align}
	Hence,
	\begin{align}
		\En_{x_t}(\widehat{r}_t - \util^*(\S_t,x_t))^2 &= \En_{x_t}\left(\sum_{i=1}^k \vtil_{i,t}\sigma(\theta^\tr_t P_i x_t) - \vtil_{i,t}\sigma(\theta_*^\tr P_i x_t)  \right)^2 \\
		&\leq \frac{2k \lambda_{\textsf{max}}\left(\sum_{i=1}^k P_i \Sigma P_i^\tr\right)}{\lambda_{\text{min}}} \En_{x_t,r_t} \inner{g_t, \theta_t-\theta_*}.
	\end{align}
	Proceeding as in the proof of Proposition~\ref{prop:isotron_single} completes the proof.
\end{proof}

\bibliographystyle{alpha}
\bibliography{ref}

\end{document}